\documentclass[lettersize,journal,twoside]{IEEEtran}
\usepackage{amsmath, amssymb, amsfonts, wasysym}
\usepackage{graphicx, wrapfig}
\usepackage{float}
\usepackage[ruled,vlined,linesnumbered]{algorithm2e}
\usepackage{booktabs}
\usepackage[font=small]{caption}
\usepackage{xcolor,colortbl}
\usepackage[colorlinks=true, pdfborder={0 0 0}, linkcolor=magenta]{hyperref}

\usepackage{enumerate}
\usepackage{enumitem}

\usepackage{amsthm}
\usepackage{multirow}
\usepackage{mathrsfs}
\usepackage{tikz}
\usetikzlibrary{arrows.meta,positioning,shapes.geometric,fit,backgrounds}
\usepackage{pgfplots}

\newcommand{\bb}[1]{\boldsymbol{#1}}
\SetKw{And}{\textbf{and}}
\SetKw{Or}{\textbf{or}}
\SetKw{Not}{\textbf{not}}
\SetKw{Is}{\textbf{is}}
\SetKw{True}{\textbf{true}}
\SetKw{False}{\textbf{false}}
\SetKw{Return}{\textbf{return}}
\SetKw{Reached}{\textbf{reached}}
\SetKw{Advanced}{\textbf{advanced}}
\SetKw{Trapped}{\textbf{trapped}}
\SetKw{Continue}{\textbf{continue}}
\SetKw{Break}{\textbf{break}}
\SetKwRepeat{Do}{do}{while}

\usepackage{etoolbox}
\makeatletter
\patchcmd{\@algocf@start}
{-1.5em}
{0pt}
{}{}
\makeatother

\newtheorem{theorem}{Theorem}

\newtheorem{problem}{Problem}[section]
\newtheorem{definition}{Definition}[section]

\theoremstyle{definition}

\title{\LARGE \bf
	Online Generation of Collision-Free Trajectories\\
	in Dynamic Environments
}

\author{Nermin Covic$^{1}$ and Bakir Lacevic$^{1}$
\thanks{Manuscript received: February 28, 2026; 
	Revised: May 22, 2026;
	Accepted: June 28, 2026. \\
	This paper was recommended for publication by Editor Aniket Bera upon evaluation of the Associate Editor and Reviewers' comments.\\
	$^{1}$Authors are with the Faculty of Electrical Engineering, University of Sarajevo, Bosnia and Herzegovina, {\tt\small \{nermin.covic, bakir.lacevic\}@etf.unsa.ba}.\\
	Digital Object Identifier (DOI): see top of this page.}%
}

\graphicspath{{Figures/}} 

\begin{document}
\maketitle

\begin{abstract}
In this paper, we present an online method for converting an arbitrary geometric path, represented by a sequence of states, and generated by any planner (e.g., sampling-based planners such as RRT or PRM, search-based planners such as ARA*, etc.), into a kinematically feasible, jerk-limited trajectory. The method generates a sequence of quintic/quartic splines that can be discretized at a user-specified control rate and streamed to a low-level robot controller. Our approach enables real-time adaptation to environmental changes and can be re-invoked at any instant to generate a new trajectory from the robot's current state to a desired target state or sequence of states. Under a bounded-obstacle-velocity assumption, the method provides conditional stopping-safety guarantees over a finite time interval in dynamic environments, while allowing bounded geometric deviation from the original path. Kinematic constraints, including jerk limits, are explicitly considered. We validate the approach in a comparative simulation study against a competing method, demonstrating favorable behavior w.r.t. smoothness, computational time, and real-time performance, particularly with frequent target-state changes (up to $1\,\mathrm{[kHz]}$). Real-robot experiments demonstrate applicability in real-world scenarios, including scenarios with a human as an obstacle.
\end{abstract}

\vspace{-0.5cm}
\section{Introduction}
\label{Sec. Introduction}
Time-optimal path parameterization (TOPP) computes a minimum-time scaling of a collision-free path under dynamic constraints. Classical offline methods, including Bobrow's algorithm \cite{bobrow1985time}, TOPP \cite{shin1985minimum}, and TOPP-RA \cite{pham2018new}, rely on \textit{Pontryagin's maximum principle} or \textit{convex optimization} to generate globally optimal trajectories for predefined paths. 
Numerical-integration methods are fast but difficult to implement robustly, while convex-optimization approaches are more stable but heavier. However, these techniques are offline and assume full path/environment knowledge, limiting real-time adaptability.

In modern applications, manipulators must often replan \textit{on-the-fly} in dynamic environments (DEs) (e.g., human-robot collaboration, moving obstacles, or target changes), while guaranteeing kinematic feasibility (position ($\mathbb{P}$), velocity ($\mathbb{V}$), acceleration ($\mathbb{A}$), and jerk ($\mathbb{J}$) constraints) and safety (collision avoidance, joint limits, etc.). Hence, \textit{online trajectory generation} (OTG) methods are briefly surveyed. 

Several works focus on smooth, jerk-constrained trajectory generation. Quintic-polynomial methods produce non-oscillatory, near time-optimal motions with bounded computation by joining fifth-order polynomials between waypoints and designing ramp conditions \cite{macfarlane2003jerk}. Continuous-jerk online generators based on multi-segment sine jerk profiles ensure smooth transitions under $\mathbb{V}$ and $\mathbb{A}$ limits \cite{zhao2022jerk}. To address the non-convexity of third-order constraints, \cite{lee2024performance} formulates conservative inequalities and solves the resulting problem via $n$-dimensional sequential linear programming.

Real-time and sensor-reactive OTG strategies emphasize synchronization and efficiency. An S-curve-based method generates multi-DoF synchronized trajectories while minimizing either $\mathbb{V}$ or $\mathbb{A}$ peaks \cite{wang2020research}. An online polytope-algebra approach exploits the robot's full kinematic capabilities by repeatedly computing a time-optimal trapezoidal acceleration profile over the remaining path \cite{skuric2025online}. Likewise, a segment-based adaptive look-ahead feedrate scheduler using a \textit{local dynamic window} and \textit{maximum velocity curve} balances efficiency and stability \cite{xu2024segmented}. Path-accurate generation under $\mathbb{V}$, $\mathbb{A}$, and $\mathbb{J}$ constraints is addressed in \cite{lange2015path} using forward scaling and backtracking, preventing segment blending through arc-length interpolation and enabling execution at each sampling step.

Safety-aware and collision-free planning approaches extend OTG to DEs. The framework in \cite{shao2024online} combines trajectory optimization with online local replanning to generate smooth and time-efficient manipulator trajectories around unforeseen dynamic obstacles. Complementary, the fast and safe trajectory-planning approach in \cite{palleschi2021fast} preserves a specified geometric path while replanning its temporal profile online to satisfy dynamically updated safety and kinodynamic constraints. Orthogonal collocation with geometric collision modeling enables smooth, time- and jerk-optimal collision-free trajectories through prescribed waypoints \cite{wen2022path}. Projected path dynamics are combined with reach-avoid safe sets in \cite{mcgovern2024safe} to compute admissible velocity and torque profiles while satisfying state/input constraints and temporal specifications. Autonomous reachability-based manipulator trajectory design \cite{holmes2020reachable} provides formal safety guarantees through offline reachable sets and provably correct online collision constraints, with fail-safe maneuvers. A two-layer human-robot collaboration architecture adapts nominal kinodynamic trajectories to human motion while enforcing safety via velocity scaling or replanning \cite{pupa2021safety}.

Optimization-based methods are another important class of collision-free trajectory planners. CHOMP \cite{zucker2013chomp} formulates continuous trajectory optimization by iteratively improving an initial trajectory w.r.t. smoothness and obstacle-avoidance costs. TrajOpt \cite{schulman2014motion} uses sequential convex optimization and convex collision checking to generate collision-free trajectories from simple initializations. However, such planners are generally local and may depend on initial trajectory quality.

\begin{table*}[t]
	\centering
	\caption{Summary of key contributions of online trajectory generation and time parameterization methods.}
	\vspace{-0.2cm}
	\resizebox{\textwidth}{!}{%
	\begin{tabular}{@{}lllllllll@{}}
		\toprule
		Method &
		Year / Venue &
		Constraints &
		\begin{tabular}[c]{@{}l@{}}
			Online/\\
			Offline
		\end{tabular} &
		\begin{tabular}[c]{@{}l@{}}
			Envir.\\ 
			type
		\end{tabular} &
		\begin{tabular}[c]{@{}l@{}}
			Safety/\\
			Coll. aware
		\end{tabular} &
		Optimality &
		\begin{tabular}[c]{@{}l@{}}
			Robot type (Validation in simulation (S) / \\
			Validation in real-world experiments (E)) 
		\end{tabular} &
		\begin{tabular}[c]{@{}l@{}}
			Comput.\\ 
			time in $\mathrm{[s]}$
		\end{tabular}
		 \\
		\midrule
		
		\rowcolor[HTML]{EFEFEF}
		\begin{tabular}[c]{@{}l@{}}
			Bobrow's \cite{bobrow1985time},\\
			TOPP \cite{shin1985minimum}
		\end{tabular}
		&
		\begin{tabular}[c]{@{}l@{}}
			1985 / IJRR, \\
			1985 / T-AC 
		\end{tabular}
	    & 
		dynamic (torque) & 
		offline & 
		static & 
		no & 
		path-constrained, time-optimal & 
		none (numerical examples) & 
		n/a \\
		
		TrajOpt \cite{schulman2014motion} 
		&2014 / IJRR 
		&kinematic ($\mathbb{P}$, $\mathbb{V}$, $\mathbb{A}$, $\mathbb{J}$)
		&offline 
		&static 
		&yes 
		&locally optimal via sequential convex opt. 
		&Atlas humanoid robot (S), mobile robot PR2 (S\&E) 
		&$\sim 10^{-1}$ \\
		
		\rowcolor[HTML]{EFEFEF}
		TOPP-RA \cite{pham2018new} &
		2018 / T-RO & 
		kin. ($\mathbb{V}$, $\mathbb{A}$), dyn. (torque) & 
		offline & 
		static & 
		no & 
		path-constrained, time-optimal & 
		6-DoF arm (S), 50-DoF legged robot (S) & 
		$\sim 10^{-2}$ \\
		
		ARMTD \cite{holmes2020reachable} &
		2020 / RSS & 
		kinematic ($\mathbb{P}$, $\mathbb{V}$, $\mathbb{A}$) & 
		hybrid & 
		dynamic & 
		yes & 
		arbitrary user-specified cost & 
		Fetch mobile manipulator (S\&E) & 
		$\sim 10^{-1}$  \\
		
		\rowcolor[HTML]{EFEFEF}
		Pupa's \cite{pupa2021safety} &
		2021 / RA-L & 
		kinematic ($\mathbb{P}$, $\mathbb{V}$, $\mathbb{A}$) & 
		online & 
		dynamic & 
		yes & 
		max. admissible speed along the path & 
		Pilz PRBT 6-DoF arm (S\&E) & 
		$\sim 10^{-3}$ \\
		
		Ruckig \cite{berscheid2021ruckig} &
		2021 / RSS & 
		kinematic ($\mathbb{P}$, $\mathbb{V}$, $\mathbb{A}$, $\mathbb{J}$) & 
		online & 
		dynamic & 
		no & 
		time-optimal & 
		Franka Panda 7-DoF arm (S\&E) &
		$\sim 10^{-5}$  \\
		
		\rowcolor[HTML]{EFEFEF}
		Zhao's \cite{zhao2022jerk} &
		2022 / IROS & 
		kinematic ($\mathbb{V}$, $\mathbb{A}$, $\mathbb{J}$) & 
		online & 
		dynamic & 
		no & 
		time-optimal within the sinusoidal-$\mathbb{J}$ family & 
		UR3 arm (S) & 
		n/a \\
		
		CuRobo \cite{sundaralingam2023curobo} &
		2023 / ICRA & 
		kinematic ($\mathbb{P}$, $\mathbb{V}$, $\mathbb{A}$, $\mathbb{J}$) & 
		online & 
		dynamic & 
		yes & 
		minimum-$\mathbb{J}$ \& minimum-$\mathbb{A}$ (locally) & 
		UR5e (S), UR10 (S), Kinova (S), Jetson AGX (E) & 
		$\sim 10^{-2}$ \\
		
		\rowcolor[HTML]{EFEFEF}
		McGovern's \cite{mcgovern2024safe} 
		&2024 / T-RO 
		&dyn. (torque), state 
		&hybrid 
		&dynamic 
		&yes 
		&safe feasible profiles 
		&UR10 (E) 
		&$\sim 10^{-4}$ \\
		
		Skuric's \cite{skuric2025online} 
		&2025 / T-RO 
		&kinematic ($\mathbb{P}, \mathbb{V}, \mathbb{A}, \mathbb{J}$) 
		&online 
		&dynamic
		&no 
		&near time-optimal 
		&Franka Panda (S\&E)  
		&$\sim 10^{-3}$ \\
		
		\rowcolor[HTML]{EFEFEF}
		Patra's \cite{patra2025kinodynamic} &
		2025 / JMR & 
		kinodynamic ($\mathbb{P}$, $\mathbb{V}$, $\mathbb{A}$) & 
		hybrid & 
		dynamic & 
		yes & 
		control efforts in a receding-horizon (locally) & 
		mobile manipulator (S\&E) & 
		$\sim 10^{-1}$  \\
		
		\textbf{CFS45 (ours)} &
		\textbf{-- / --} & 
		\textbf{kinematic ($\mathbb{P}$, $\mathbb{V}$, $\mathbb{A}$, $\mathbb{J}$)} & 
		\textbf{online} & 
		\textbf{dynamic} & 
		\textbf{yes} & 
		\textbf{time-optimal within the parabolic-$\mathbb{J}$ family} & 
		\textbf{planar 2-DoF (S), UFACTORY xArm6 (S\&E)} & 
		\textbf{$\sim 10^{-6}$ } \\
		
		\bottomrule
	\end{tabular}
	}
	\label{tab_sota_summary}
	\vspace{-0.7cm}
\end{table*}

From a computational perspective, CuRobo \cite{sundaralingam2023curobo} uses GPU acceleration to compute collision-free, minimum-jerk trajectories in tens of $\mathrm{[ms]}$ through parallel IK, collision checking, and trajectory optimization. In contrast, \cite{patra2025kinodynamic} proposes online receding-horizon planning for multiple mobile manipulators in DEs, jointly optimizing base and arm motion under kinodynamic and collision constraints.

Ruckig \cite{berscheid2021ruckig} is an open-source OTG library for real-time, jerk-limited, time-synchronized multi-DoF trajectory generation with arbitrary initial/target states and asymmetric limits. Due to its efficiency, robustness, and adoption in MoveIt, CoppeliaSim, and Frankx, Ruckig is our primary benchmark.

Tab. \ref{tab_sota_summary} summarizes the above papers. Offline methods remain the most suitable when the full path and environment are known in advance, and the goal is globally optimized path parameterization, whereas online methods trade some optimality for reactivity. Together, these approaches push toward the goal of real-time, guaranteed-safe trajectory generation.

This paper focuses on online time-parameterization and OTG methods that extend classical ideas to dynamic, collision-critical contexts. We highlight the following contributions:
\begin{itemize}
	\item A new method for fast conversion of geometric paths to jerk-limited trajectories with safety guarantees in both static and dynamic environments.
	
	\item The proposed spline-based approach can operate with either one or more changing target waypoints. 
	
	\item The method can handle arbitrary admissible initial conditions, and non-zero final velocity and acceleration values.
\end{itemize}

The next sections are organized as follows. Sec. \ref{Sec. Problem Statement and Assumptions} introduces the problem, while Sec. \ref{Sec. Method for Spline Computation} describes the proposed method for local trajectory computation. Sec. \ref{Sec. Path-to-Trajectory Conversion} explains how geometric path can be generally converted to a corresponding trajectory. Sec. \ref{Sec. Trajectory Collision Checking} describes collision checking procedure for the candidate trajectory. Sec. \ref{Sec. Simulation Study} deals with a comprehensive simulation study in which the proposed method is compared to a state-of-the-art algorithm. Afterward, Sec. \ref{Sec. Experimental Validation} validates the novel approach on a real robot. Finally, Sec. \ref{Sec. Conclusion} brings some conclusions and future work directions.

\vspace{-0.3cm}
\section{Problem Statement and Assumptions}
\label{Sec. Problem Statement and Assumptions}
Let $\mathcal{C}$ denote the robot's $n$-dimensional \textit{configuration space}. The \textit{obstacle space} $\mathcal{C}_{\mathrm{obs}}\subseteq\mathcal{C}$ is the closed set of configurations that cause collision, either with external obstacles or by self-collision, while $\mathcal{C}_{\mathrm{free}}=\mathcal{C}\setminus\mathcal{C}_{\mathrm{obs}}$ denotes the corresponding \textit{free space}. We assume that the workspace contains a finite set of possibly overlapping convex \textit{world obstacles} $\mathcal{WO}_j$, $j\in\{1,\dots,N_{\mathrm{obs}}\}$, while non-convex obstacles are handled through convex decomposition (e.g., \cite{lien2008approximate}). 
	
The robot's \textit{start}, \textit{current}, \textit{target}, and \textit{goal configurations} are denoted by $\bb{q}_{\mathrm{start}}$, $\bb{q}_{\mathrm{curr}}$, $\bb{q}_{\mathrm{target}}$, and $\bb{q}_{\mathrm{goal}}$, respectively.
A \textit{motion trajectory} $\bb{\pi}(t)=\bb{\pi}[\bb{q}_0,\bb{q}_f]$ is a continuous mapping $\bb{\pi}:[t_0,t_f]\rightarrow\mathcal{C}$ with $\bb{\pi}(t_0)=\bb{q}_0$ and $\bb{\pi}(t_f)=\bb{q}_f$, where $t_0$ and $t_f$ denote the \textit{initial} and \textit{unknown final time}, respectively. It is considered \textit{valid} if it satisfies the \textit{constraints} $\mathcal{K}$ and remains collision-free, i.e., $\bb{\pi}(t)\in\mathcal{C}_{\mathrm{free}}(t)$ for all $t\in[t_0,t_f]$.

Other assumptions are as follows: all constraints $\mathcal{K}$ are known; robot-obstacle collision/distance query is available; a nominal path or $\bb{q}_{\mathrm{target}}$ is supplied by an upstream planner; and obstacle speeds are bounded by a prescribed value $v_{\mathrm{obs}}$ in DEs. Obstacle motion directions are not assumed to be known or predictable, and obstacle accelerations need not be bounded.

\vspace{-0.15cm}
\begin{problem}
	\label{Problem path-to-trajectory}
	Let $\bb{Q}=\{\bb{q}_1,\dots,\bb{q}_N\}$ be a nominal geometric path in $\mathcal{C}$-space, where $\bb{q}_1=\bb{q}_{\mathrm{start}}$ and $\bb{q}_N=\bb{q}_{\mathrm{goal}}$. The goal is to compute a time-parameterized trajectory $\bb{\pi} : [t_0,t_f]\mapsto \mathcal{C}$ that approximates $\bb{Q}$ while satisfying the kinematic constraints $\mathcal{K}$ for each robot's $i$-th joint, $i\in\{1,\dots,n\}$, as follows:
	\vspace{-0.2cm}
	\begin{equation} \label{eq_kin_constraints}
		|\pi_i^{(o)}(t)| \leq q_{\mathrm{m}_i}^{(o)}, \quad
		o\in\{0,1,2,3\},\quad
		\forall t\in[t_0,t_f],
		\vspace{-0.2cm}
	\end{equation}
	where $q_{\mathrm{m}_i}^{(o)}$ denotes the corresponding bound on $\mathbb{P}$, $\mathbb{V}$, $\mathbb{A}$, and $\mathbb{J}$. In addition, it must be determined whether the computed trajectory is collision-free in the current (static) realization of the environment, or whether it will remain collision-free for a certain amount of time in a dynamic environment. 
\end{problem}
\vspace{-0.15cm}

Corresponding solutions to the above-defined problem are divided into Secs. \ref{Sec. Method for Spline Computation}, \ref{Sec. Path-to-Trajectory Conversion}, and \ref{Sec. Trajectory Collision Checking}, respectively.

\vspace{-0.3cm}
\section{Method for Spline Computation}
\label{Sec. Method for Spline Computation}
This section introduces the local segment-generation method simply referred to as CFS45 (Collision-Free $4^\mathrm{th}$/$5^\mathrm{th}$ order Splines). For $i$-th joint, an $m$-th order spline is assumed as
\vspace{-0.35cm}
\begin{equation}\label{eq_spline_m_th_order}
	\pi_i(t) = \phi_i^{_{(m)}} t^m + \cdots + \phi_i^{_{(1)}} t + \phi_i^{_{(0)}}, \quad t\in[t_0,t_{f_i}],
	\vspace{-0.2cm}
\end{equation}
where $\phi_i$ are the (temporary) unknown polynomial coefficients, and $t_{f_i}$ is a final time which needs to be determined.

\vspace{-0.3cm}
\subsection{Description of the Proposed Approach}
In order to satisfy \eqref{eq_kin_constraints}, it is sufficient to choose $m=5$ within \eqref{eq_spline_m_th_order} to obtain quintic spline for the robot's $i$-th joint. Quintic splines provide the lowest-degree closed-form representation satisfying boundary conditions on $\mathbb{P}$, $\mathbb{V}$, and $\mathbb{A}$, which is well suited for efficient online replanning. In principle, alternative jerk-limited formulations could also be integrated.

Initial boundary conditions at $t_0=0$ are:
\vspace{-0.2cm}
\begin{equation}\label{eq_initial_bc}
	\pi_i(0)=\phi_i^{_{(0)}},\quad \dot{\pi}_i(0)=\phi_i^{_{(1)}},\quad \ddot{\pi}_i(0)=2\phi_i^{_{(2)}},
	\vspace{-0.2cm}
\end{equation}
which are known from the previous spline, thus they
immediately yield coefficients $\phi_i^{_{(2)}}$, $\phi_i^{_{(1)}}$, and $\phi_i^{_{(0)}}$. Combining final boundary conditions at $t=t_{f_i}$
the following can be obtained:
\vspace{-0.3cm}
\begin{equation} \label{eq_c_relation}
	\resizebox{0.9\columnwidth}{!}{$\textstyle
	\phi_i^{_{(3)}} t_{f_i}^3 + \Big(3\phi_i^{_{(2)}} - \frac{\ddot{\pi}_{f_i}}{2}\Big) t_{f_i}^2 + (6 \phi_i^{_{(1)}} + 4\dot{\pi}_{f_i}) t_{f_i} + 10(\phi_i^{_{(0)}} - \pi_{f_i}) = 0,
	$}
	\vspace{-0.1cm} 
\end{equation}
\begin{equation} \label{eq_b_formula}
	\resizebox{0.85\columnwidth}{!}{$\textstyle
	\phi_i^{_{(4)}} = \frac{1}{t_{f_i}^3}\left[-\tfrac{3}{2} \phi_i^{_{(3)}} t_{f_i}^2 + \left(-\tfrac{3}{2} \phi_i^{_{(2)}} - \tfrac{\ddot{\pi}_{f_i}}{4}\right) t_{f_i} - \phi_i^{_{(1)}} + \dot{\pi}_{f_i} \right],  
	$}
	\vspace{-0.1cm}
\end{equation}
\begin{equation} \label{eq_a_formula}
	\resizebox{0.75\columnwidth}{!}{$\textstyle
	\phi_i^{_{(5)}} = \frac{1}{20 t_{f_i}^3}\left[-12 \phi_i^{_{(4)}} t_{f_i}^2 - 6 \phi_i^{_{(3)}} t_{f_i} - 2 \phi_i^{_{(2)}} + \ddot{\pi}_{f_i} \right],
	$}
	\vspace{-0.1cm}
\end{equation}
where the final boundary conditions are expressed as:
\vspace{-0.2cm}
\begin{equation} \label{eq_final_bc}
	\pi_{f_i} = \pi_i(t_{f_i}), \quad 
	\dot{\pi}_{f_i} = \dot{\pi}_i(t_{f_i}), \quad
	\ddot{\pi}_{f_i} = \ddot{\pi}_i(t_{f_i}).
	\vspace{-0.1cm}
\end{equation}
The jerk function for $i$-th robot's joint is given as
\vspace{-0.2cm}
\begin{equation}\label{eq_jerk}
\dddot{\pi}_i(t) = 60 \phi_i^{_{(5)}} t^2 + 24 \phi_i^{_{(4)}} t + 6 \phi_i^{_{(3)}}.
\vspace{-0.15cm}
\end{equation}
Since the following holds (see \eqref{eq_kin_constraints}):
\vspace{-0.2cm}
\begin{equation}\label{eq_c_range}
	|\dddot{\pi}_i(0)| = |6 \phi_i^{_{(3)}}| \le \dddot{q}_{\mathrm{m}_i} \implies \phi_i^{_{(3)}} \in  \Big[-\frac{\dddot{q}_{\mathrm{m}_i}}{6}, \frac{\dddot{q}_{\mathrm{m}_i}}{6}\Big],
	\vspace{-0.2cm}
\end{equation}
we can pick $\phi_i^{_{(3)}}$ from such range and then compute $t_{f_i}$ from \eqref{eq_c_relation} by solving the cubic equation (either numerically or exactly using Cardano's formula). Then, for each positive real solution $t_{f_i}$, \eqref{eq_b_formula} provides $\phi_i^{_{(4)}}$, and finally \eqref{eq_a_formula} yields $\phi_i^{_{(5)}}$.


After computing all $\phi$-coefficients and the final time $t_{f_i}$, we must verify kinematic constraints by \eqref{eq_kin_constraints}. According to \eqref{eq_initial_bc}--\eqref{eq_final_bc}, all boundary conditions for $t = 0$ and $t = t_{f_i}$ are satisfied for $\mathbb{P}$, $\mathbb{V}$, and $\mathbb{A}$. However, jerk values $\dddot{\pi}_i(0)$ and $\dddot{\pi}_i(t_{f_i})$ need to be verified. Instead of checking \eqref{eq_kin_constraints} $\forall t\in(0,t_{f_i})$, it suffices to verify at candidate extremal times, i.e., to solve $\dot{\pi}_i(t)=0$, $\ddot{\pi}_i(t)=0$, $\dddot{\pi}_i(t)=0$, and $\pi_i^{_{(4)}}(t)=0$ to obtain time instances $t_{\mathrm{m},\mathrm{pos}}$, $t_{\mathrm{m},\mathrm{vel}}$, $t_{\mathrm{m},\mathrm{acc}}$ and $t_{\mathrm{m},\mathrm{jerk}}$ of maximal $\mathbb{P}$, $\mathbb{V}$, $\mathbb{A}$, and $\mathbb{J}$, respectively, and evaluate constraints there. If any constraint is violated, the guess for $\phi_i^{_{(3)}}$ must be changed (as it will be described in Subsec. \ref{Subsec. Jerk Computation}), after which $t_{f_i}$, $\phi_i^{_{(4)}}$, and $\phi_i^{_{(5)}}$ must be computed again until \eqref{eq_kin_constraints} holds. In case the synchronization of all joints is desired, the final time is chosen as $t_f = \max\{t_{f_1}, \dots, t_{f_n}\}$, upon which all coefficients $\phi_i^{_{(5)}}$, $\phi_i^{_{(4)}}$, and $\phi_i^{_{(3)}}$ are recomputed and all $\mathcal{K}$ verified. We adopt a simple baseline synchronization strategy to ensure deterministic implementation and straightforward per-joint feasibility verification, rather than aiming for an optimal coupling mechanism.

Finally, it is worth stressing that the same logic is utilized to compute quartic splines (for $m=4$ within \eqref{eq_spline_m_th_order}), which are used for the robot's emergency stopping when computing safe trajectories (see Sec. \ref{Sec. Trajectory Collision Checking}). The only difference is that the final position $\pi_i(t_{f_i})$ is free and is computed as a consequence.

\vspace{-0.2cm}
\subsection{Jerk Computation}
\label{Subsec. Jerk Computation}
\vspace{-0.05cm}

\begin{theorem}[Jerk computation]
	\label{Thm. Jerk Computation}
	Let the coefficient $c_i = \phi_i^{_{(3)}}$ and let
	$I_i^{(0)}
	=
	\left[
	c_{i,\mathrm{left}}^{(0)},
	c_{i,\mathrm{right}}^{(0)}
	\right]
	=
	\left[
	-\frac{\dddot{q}_{\mathrm{m}_i}}{6},
	\frac{\dddot{q}_{\mathrm{m}_i}}{6}
	\right]
	$
	be the admissible interval for $c_i$. Alg. \ref{Pseudocode Jerk Computation} first checks the boundary values $c_{i,\mathrm{left}}^{(0)}$ and $c_{i,\mathrm{right}}^{(0)}$ (line \ref{Alg1 line2}). If at least one boundary value yields real candidates for $t_{f_i}$ satisfying $\mathcal{K}$, the algorithm immediately returns the feasible pair $(c_i^*,t^*)$ with the shortest final time (line \ref{Alg1 line5}). Otherwise, provided that the boundary check admits the existence of real candidates but $\mathcal{K}$ is not satisfied at the boundary, the algorithm applies bisection and converges to a feasible value of $c_i$ with the precision $\Delta c_i$ (lines \ref{Alg1 line while}--\ref{Alg1 line while end}).
\end{theorem}

\vspace{-0.3cm}
\begin{proof}
	If the boundary check yields a feasible candidate, the claim follows directly since Alg. \ref{Pseudocode Jerk Computation} selects the one with the shortest final time. Otherwise, the bisection step repeatedly discards one half of the current admissible interval, while the retained half preserves the possibility of satisfying $\mathcal{K}$. Hence, the interval width decreases as
	$
	c_{i,\mathrm{right}}^{(k)}
	-
	c_{i,\mathrm{left}}^{(k)}
	=
	(c_{i,\mathrm{right}}^{(0)}
	- c_{i,\mathrm{left}}^{(0)}) /	
	2^k,
	$
	so the procedure terminates after a finite number of iterations once this width becomes smaller than $\Delta c_i$ (e.g., we use $0.001 \cdot 6 \,|I_i^{(0)}|$).
\end{proof}
\vspace{-0.1cm}

\noindent\textit{Remark.} Theorem \ref{Thm. Jerk Computation} states sufficient conditions for the solution existence. Otherwise, no solution is returned, and the planner reuses the \scalebox{0.96}[1]{already-computed} trajectory from previous iterations.

\SetEndCharOfAlgoLine{}
\begin{algorithm}[t]
	\small
	\setlength{\baselineskip}{9pt}
	
	\caption{Jerk Computation}
	\label{Pseudocode Jerk Computation}
	
	\KwIn{$c_{i,\mathrm{left}}^{(0)}=-\frac{\dddot{q}_{\mathrm{m}_i}}{6}$, $c_{i,\mathrm{right}}^{(0)}=\frac{\dddot{q}_{\mathrm{m}_i}}{6}$, $\Delta c_i$}
	
	\KwOut{$c_i^*$, $t^*$}

	$c_i^*\gets \varnothing$, \quad $t^*\gets \varnothing$, \quad $k \gets 1$ \\
	
	$t_{f_i} \gets \mathtt{computeRealCandidates}(c_{i,\mathrm{left}}^{(0)},\, c_{i,\mathrm{right}}^{(0)})$ \label{Alg1 line2}\\

	\If{\emph{exists feasible $t_{f_i}$}}
	{
		\If{\emph{satisfies $\mathcal{K}$ for $c_{i,\mathrm{left}}^{(0)}$ or $c_{i,\mathrm{right}}^{(0)}$}}
		{
			$t^*,\, c_i^*\gets \mathtt{getMinimalFeasible}(t_{f_i},\, c_i)$ \label{Alg1 line5}\\
			
			\Return{$c_i^*$, $t^*$}
		}
	}
	\Else
	{
		\Return{\emph{No solution can be found!}}
	}
	
	\While{$c_{i,\mathrm{right}}^{(k-1)} - c_{i,\mathrm{left}}^{(k-1)} > \Delta c_i$} 
	{\label{Alg1 line while}
		
		$c_i^{(k)} \gets \left(c_{i,\mathrm{left}}^{(k-1)} + c_{i,\mathrm{right}}^{(k-1)}\right)/2$ \label{Alg1 line14}\\
		
		$c_{i,\mathrm{left}}^{(k)} \gets c_{i,\mathrm{left}}^{(k-1)}$, \quad $c_{i,\mathrm{right}}^{(k)} \gets c_i^{(k)}$\\
		
		$t_{f_i} \gets \mathtt{computeRealCandidates}(c_i^{(k)})$\\	
		
		\If{\emph{exists feasible $t_{f_i}$}}
		{
			\If{\emph{satisfies $\mathcal{K}$}}
			{
				$t^*,\, c_i^*\gets \mathtt{getMinimalFeasible}(t_{f_i},\, c_i^{(k)})$\\
			}
			\Else
			{
				$c_{i,\mathrm{left}}^{(k)} \gets c_i^{(k)}$, \quad $c_{i,\mathrm{right}}^{(k)} \gets c_{i,\mathrm{right}}^{(k-1)}$\\
			}
		}
		
		$k \gets k + 1$ \label{Alg1 line while end}\\
	}
	
	\Return{$c_i^*$, $t^*$}
\end{algorithm}
\setlength{\textfloatsep}{0pt}

\section{Path-to-Trajectory Conversion}
\label{Sec. Path-to-Trajectory Conversion}
\vspace{-0.1cm}
This section describes how CFS45 solves Problem \ref{Problem path-to-trajectory} by converting an arbitrary, preferably collision-free, geometric path $\bb{q}_{\mathrm{start}}\rightarrow\bb{q}_{\mathrm{goal}}$ into a time-parameterized trajectory represented as a sequence of splines $\bb{\Pi}$, where each spline is computed within Sec. \ref{Sec. Method for Spline Computation}, such that any constraint from \eqref{eq_kin_constraints} must not be violated. To facilitate the understanding of this process, we will refer to Alg. \ref{Pseudocode Path-to-Trajectory Conversion} and Fig. \ref{fig_path2trajectory} in the sequel. The geometric path $\bb{Q}: \bb{q}_{\mathrm{start}}\rightarrow \bb{q}_{\mathrm{goal}}$ can be computed by any planner (e.g., RRT-Connect \cite{kuffner2000rrt}, RGBMT* \cite{covic2023asymptotically}, etc.). 

\begin{figure*}[t]
	\centering
	\includegraphics[width=0.99\linewidth]{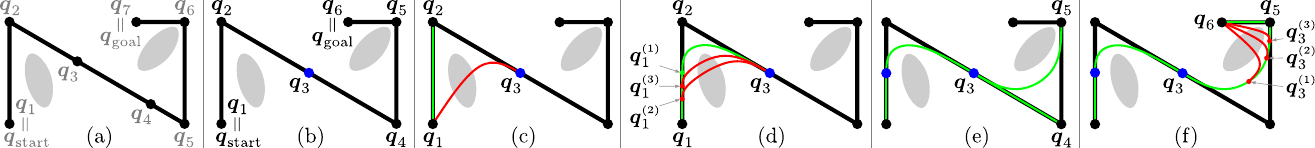}
	\vspace{-0.1cm}
	\caption{Example of the process of converting a path $\bb{Q}: \bb{q}_{\mathrm{start}}\rightarrow \bb{q}_{\mathrm{goal}}$ (black) to a corresponding trajectory (green). See the text for details.}
	\label{fig_path2trajectory}
	\vspace{-0.65cm}
\end{figure*}

First, the procedure begins with line \ref{Alg2 line1} by reallocating nodes on the path such that the Euclidean distance in $\mathcal{C}$-space between two consecutive nodes does not exceed $D_{\max}$. The goal is to reduce the number of path nodes without affecting much the path geometry (e.g., intermediate nodes $\bb{q}_3$ and $\bb{q}_4$ from a straight-line segment are replaced with a new $\bb{q}_3$, as depicted in blue in Fig. \ref{fig_path2trajectory} (a)). If CFS45 is applied to a real-time planning algorithm such as RRT$^\mathrm{X}$ \cite{otte2016rrtx} or DRGBT \cite{covic2025realtime}, a convenient choice is  $D_{\max}=\|\dot{\bb{q}}_\mathrm{m}\|\,T$, where $T$ is the \textit{planner iteration time}. This choice provides a practical trade-off between smaller values (which increase the number of generated splines and runtime) and larger values (which increase the deviation from the original geometric path).

Afterward, the sequence of splines $\bb{\Pi}$ is being computed within lines \ref{Alg2 line while}--\ref{Alg2 line end}. At each iteration, we attempt to compute a spline from $\bb{q}_k$ toward $\bb{q}_{k+2}$, $k \in \{1,\dots,N-2\}$, as an attempt to interpolate the path (if possible) around the \textit{corner point} $\bb{q}_{k+1}$. It is worth indicating that $\mathbb{V}$ and $\mathbb{A}$ at the black points in Fig. \ref{fig_path2trajectory} are always zero, while $\mathbb{V}$ at the blue points can be generally non-zero since they lie on a straight-line segment. The $\mathbb{V}$ at the blue point $\bb{q}_k$ is estimated as $\dot{\bb{q}}_k = (\bb{q}_{k+1}-\bb{q}_k)/t$, where $t$ is picked from the range $(0, T]$ in order to keep $\dot{\bb{q}}_k$ feasible (e.g., used within lines \ref{Alg2 line5} and \ref{Alg2 line10}).

First, line \ref{Alg2 line5} computes the spline $\bb{\pi}_{k,k+2} = \bb{\pi}[\bb{q}_k,\bb{q}_{k+2}]$ (e.g., the red line in Fig. \ref{fig_path2trajectory} (c) for $k=1$) and checks it for collision by line \ref{Alg2 line6}. If it is collision-free, it will be stored in $\bb{\Pi}$, and line \ref{Alg2 line5} will compute the next spline (e.g., $\bb{\pi}_{3,5}$ in Fig. \ref{fig_path2trajectory} (e)). Otherwise, line \ref{Alg2 line10} computes the spline $\bb{\pi}_{k,k+1} = \bb{\pi}[\bb{q}_k,\bb{q}_{k+1}]$ (e.g., $\bb{\pi}_{1,2}$ in Fig. \ref{fig_path2trajectory} (c)), which is always feasible since boundary $\mathbb{V}$ and $\mathbb{A}$ are satisfied, i.e., zero. 

In case $\bb{\pi}_{k,k+2}$ is in collision (e.g., $\bb{\pi}_{1,3}$ by Fig. \ref{fig_path2trajectory} (c)), line \ref{Alg2 line11} searches for the collision-free spline $\bb{\pi}[\bb{q}_k,\bb{q}_{k+2}]$, so-called \textit{interpolating spline}, that lies between $\bb{\pi}_{k,k+1}$ and $\bb{\pi}_{k,k+2}$, which is as close as possible to $\bb{\pi}_{k,k+2}$. We utilize the bisection method as follows. The first \textit{intermediate node} is selected as $\bb{q}_{k}^{(1)} = \bb{\pi}_{k,k+1}(t_{k}^{(1)})$, where $t_{k}^{(1)} = \frac{t_k+t_{k+1}}{2}$, $\bb{q}_k = \bb{\pi}_{k,k+1}(t_k)$, and $\bb{q}_{k+1} = \bb{\pi}_{k,k+1}(t_{k+1})$. Then, the spline $\bb{\pi}[\bb{q}_{k}^{(1)}, \bb{q}_{k+2}]$ is checked for collision (e.g., Fig. \ref{fig_path2trajectory} (d) for $\bb{\pi}[\bb{q}_{1_1}, \bb{q}_3]$). If it is collision-free, we seek for the second intermediate node $\bb{q}_{k}^{(2)} = \bb{\pi}_{1,2}(t_{k}^{(2)})$, where $t_{k}^{(2)} = \frac{t_k+t_{k}^{(1)}}{2}$. Otherwise, $t_{k}^{(2)} = \frac{t_{k}^{(1)}+t_{k+1}}{2}$. Hence, the procedure is repeated until the change of $t_k^{(\cdot)}$ (e.g., $\big|t_k^{(2)}-t_k^{(1)}\big|$) becomes less than a specified threshold.

\SetEndCharOfAlgoLine{}
\begin{algorithm}[t]
	\small
	\setlength{\baselineskip}{9pt}
	\caption{Path-to-Trajectory Conversion}
	\label{Pseudocode Path-to-Trajectory Conversion}
	
	\KwIn{Geometric path $\bb{Q} = \{\bb{q}_1,\dots,\bb{q}_N\}$, $D_{\max}$}
	
	\KwOut{Spline sequence $\bb{\Pi}$}
	
	$\bb{Q} \gets \mathtt{simplify\&Densify}(\bb{Q})$ \label{Alg2 line1} \tcp{s.t.\,$\|\bb{q}_{k+1}-\bb{q}_k\|\le D_{\max}$}
	
	$\bb{\Pi} \gets \varnothing$,\quad $k \gets 1$\\
	
	\While{$k < \mathtt{size}(\bb{Q})-1$}
	{\label{Alg2 line while}
		
		$\bb{\pi}_{k,k+2} \gets \mathtt{computeSpline}(\bb{q}_k,\, \bb{q}_{k+2},\, \dot{\bb{q}}_{k+2})$ \label{Alg2 line5}\\
		
		\If{\emph{$\bb{\pi}_{k,k+2}$ is collision-free}}
		{\label{Alg2 line6}
			
			$\bb{\Pi} \gets \bb{\Pi} \cup \bb{\pi}_{k,k+2}$,\quad 
			$k \gets k+2$\\
		}
		\Else
		{
			$\bb{\pi}_{k,k+1} \gets \mathtt{computeSpline}(\bb{q}_k,\, \bb{q}_{k+1},\, \dot{\bb{q}}_{k+1})$ \label{Alg2 line10}\\
			
			$\bb{\pi}_{k,k+2} \gets \mathtt{bisection}(\bb{\pi}_{k,k+1},\, \bb{q}_{k+2})$ \label{Alg2 line11}\\
			
			\If{\emph{$\bb{\pi}_{k,k+2}$ is found}}
			{
				$\bb{\Pi} \gets \bb{\Pi} \cup \bb{\pi}_{k,k+2}$, \quad				
				$k \gets k+2$\\
			}
			\Else
			{
				$\bb{\Pi} \gets \bb{\Pi} \cup \bb{\pi}_{k,k+1}$, \quad				
				$k \gets k+1$\\
			}
		}
	}
	
	\If{$k = \mathtt{size}(\bb{Q})-1$} 
	{
		$\bb{\pi}_{k,k+1} \gets \mathtt{computeSpline}(\bb{q}_k,\, \bb{q}_{k+1},\, \dot{\bb{q}}_{k+1} = \bb{0})$ \label{Alg2 line19}\\
		
		$\bb{\Pi} \gets \bb{\Pi} \cup \bb{\pi}_{k,k+1}$ \label{Alg2 line end}\\
	}
	
	\Return{$\bb{\Pi}$}
\end{algorithm}
\setlength{\textfloatsep}{0pt}

Fig. \ref{fig_path2trajectory} (e) shows the computed splines $\bb{\pi}[\bb{q}_3,\bb{q}_4]$ and $\bb{\pi}[\bb{q}_3,\bb{q}_5]$. They are both feasible, thus the second one is chosen. Since $\bb{q}_5$ is the corner point, we seek for an interpolating spline toward $\bb{q}_6$ using the bisection method, as depicted by Fig. \ref{fig_path2trajectory} (f). If no collision-free interpolating spline is found, $\bb{q}_5$ becomes an \textit{unresolved corner node}. Consequently, the robot moves to $\bb{q}_5$, stops there, and then changes direction to follow the spline $\bb{\pi}[\bb{q}_5,\bb{q}_6]$ (computed by line \ref{Alg2 line19} since $\bb{q}_6 = \bb{q}_\mathrm{goal}$ in Fig. \ref{fig_path2trajectory} (f)). Therefore, a solution always exist and the robot can successfully reach the goal. In the worst-case scenario (i.e., when the interpolation is not possible), the obtained trajectory (shown in green) will geometrically correspond to the initial path (depicted in black). Zero velocity is enforced only at unresolved corner nodes. Otherwise, intermediate nodes may be bypassed with non-zero velocity through the constructed interpolating spline.

The procedure in Alg. \ref{Pseudocode Path-to-Trajectory Conversion} represents a general framework suitable for static environments. It can also be applied in DEs if the used dynamic planner is able to timely supply a sequence of targets (waypoints). However, in dynamic settings, it is often computationally inefficient to track multiple future targets because they generally move in time according to changes in the environment. For the sake of meeting real-time constraints, it is sometimes necessary to account only for a single target $\bb{q}_\mathrm{target}$ (the immediate next one) even in case when a fully predefined path to the goal is available.

For instance, to enable smooth trajectory interpolation, DRGBT \cite{covic2025realtime} does not wait for the robot to reach $\bb{q}_\mathrm{target}$ exactly. Instead, once the robot approaches this configuration within an Euclidean distance defined by $n\cdot R\cdot \|\dot{\bb{q}}_\mathrm{curr}\| / \|\dot{\bb{q}}_\mathrm{m}\|$, where $R$ is a user-defined parameter, the target is considered reached, thus DRGBT is allowed to generate a new target. Thus, in the next iteration, the new trajectory may interpolate around the previous target, rather than enforcing a full stop at that configuration (see, e.g., the interpolation around $\bb{q}_4$ in Fig. \ref{fig_path2trajectory} (e)). Generally, the target may be updated at each iteration if a dynamic planner decides that such a change is necessary.

\begin{figure*}[t]
	\centering
	\includegraphics[width=0.99\linewidth]{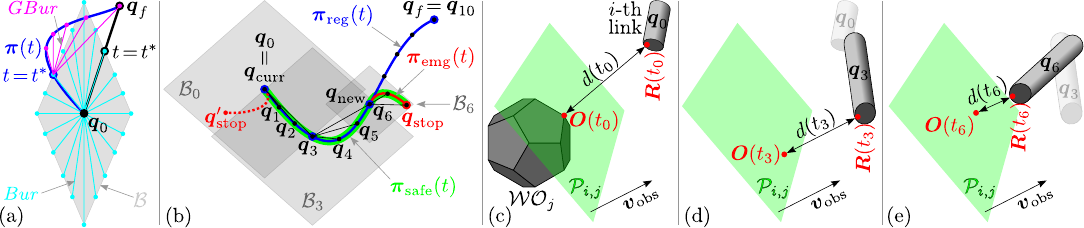}
	\vspace{-0.15cm}
	\caption{(a) Bubble $\color{gray}\mathcal{B}$ and bur $\color{cyan}Bur$ at the configuration $\bb{q}_0$ with the spine $\overline{\bb{q}_0 \bb{q}_f}$ and a local trajectory $\color{blue}\bb{\pi}(t)$. Generalized bur $\color{magenta}GBur$ originates in a configuration when $\color{blue}t=t^*$;
		(b) Process of computing $GBur$ (black lines); 
		(c,d,e) The nearest points, $\color{red}\bb{R}(t) \equiv \bb{R}_{i,j}(t)$ and $\color{red}\bb{O}(t) \equiv \bb{O}_{i,j}(t)$, from $i$-th robot's link and $\mathcal{WO}_{j}$, with the corresponding plane $\color{green}\mathcal{P}_{i,j}\equiv \mathcal{P}_{i,j}(t)$ and the distance $d(t)\equiv d_{i,j}(t)$ after $t\in\{t_0,t_3,t_6\}$.}
	\label{fig_computing_gbur}
	\vspace{-0.65cm}
\end{figure*}

\vspace{-0.3cm}
\section{Trajectory Collision Checking} 	
\label{Sec. Trajectory Collision Checking}
This section addresses Problem \ref{Problem path-to-trajectory} by explaining how a single spline from the sequence $\bb{\Pi}$ can be checked for collision. For this purpose, we utilize bubbles of free $\mathcal{C}$-space that are briefly recalled at first. Afterward, their use for generating both regular and safe trajectories is discussed.

\vspace{-0.4cm}
\subsection{Bubbles and Burs of Free $\mathcal{C}$-space}
\label{Subsec. Bubbles and Burs of Free C-space}
The concept of \textit{bubbles of free $\mathcal{C}$-space} was introduced in \cite{quinlan1994real} to define collision-free volumes around a configuration using a single workspace distance measurement. 

\vspace{-0.15cm}
\begin{definition}
	For a manipulator with $n$ revolute joints and a minimal robot-obstacle distance $d_c$, a bubble is defined as
	\vspace{-0.2cm}
	\begin{equation}\label{eq_bubble}
		\textstyle \mathcal{B}(\bb{q}_0, d_c) =
		\big\{ \bb{q} \ \big|\ \sum_{i=1}^n r_i |q_i - q_{0_i}| \le d_c \big\},
		\vspace{-0.2cm}
	\end{equation}
	where $r_i$ is the enclosing radius of a cylindrical volume aligned with the $i$-th joint axis and containing all subsequent links. The term $\sum_{i=1}^n r_i |q_i-q_{0_i}|$ upper-bounds the displacement of any point on the robot during the transition from $\bb{q}_0$ to $\bb{q}$. Hence, any $\bb{q}\in\mathcal{B}(\bb{q}_0,d_c)$ can be reached without collision.
\end{definition}
\vspace{-0.15cm}

\SetEndCharOfAlgoLine{}
\begin{algorithm}[t]
	\small
	\setlength{\baselineskip}{9pt}
	
	\SetAlgorithmName{Algorithm}{}
	\Indp
	\caption{Bur -- $\mathtt{computeBur}$}
	\label{Pseudocode Bur}
	\small
	
	\KwIn{$\bb{\pi}_\mathrm{reg}(t) = \bb{\pi}[\bb{q}_0,\bb{q}_f]$, $\Delta t$, $\bb{d}(t_0)$}
	
	\KwOut{$Bur$} 
	
	$\bb{Q},\, Bur \gets \varnothing$, \quad
	$k \gets 0$ \label{Alg3 begin} \\
	
	\While{$k\Delta t < \mathtt{getTime}(\bb{q}_f)$}
	{
		$\bb{Q} \gets [\bb{Q},\, \bb{\pi}_\mathrm{reg}(k\Delta t)]$, \quad		
		$k \gets k+1$ \\
	}
	
	$\bb{Q} \gets [\bb{Q},\, \bb{q}_f]$ \label{Alg3 line6} \\
	
	$\mathcal{B}_0 \gets \mathcal{B}(\bb{q}_0,\, \bb{d}(t_0))$ \label{Alg3 bubble} \\
	
	\For{$k = 1 : \mathtt{size}(\bb{Q})$}
	{\label{Alg3 for}
		
		\If{$\bb{q}_k \in \mathcal{B}_0$}
		{
			$Bur \gets [Bur,\, \overline{\bb{q}_0\bb{q}_k}]$ \\
		}
		\Else
		{
			\Return $Bur$ \label{Alg3 end} \\
		}
		
	}
\end{algorithm}
\setlength{\textfloatsep}{0pt}

Fig. \ref{fig_computing_gbur} (a) depicts a typical diamond-shaped bubble for a 2-DoF robot. We define a \textit{spine} as a portion of the ray emanating from $\bb{q}_0$ toward a remote configuration $\bb{q}_f$ that lies within $\mathcal{B}(\bb{q}_0,d_c)$. Specifically, for $\bb{q}(t)=\bb{q}_0+t(\bb{q}_f-\bb{q}_0)$, the corresponding spine is $\overline{\bb{q}_0\bb{q}(t^*)}$, where the maximal admissible value is determined by 
$t^* = \frac{d_c}{\sum_{i=1}^n r_i |q_{f_i}-q_{0_i}|}.$ 

In case of a nonlinear trajectory $\bb{\pi}(t)=\bb{\pi}[\bb{q}_0,\bb{q}_f]$, it's intersection with the bubble's border can be obtained from
$\sum_{i=1}^n r_i |\pi_i(t)-q_{0_i}| = d_c.$
Since multiple solutions may exist, the smallest positive one, $t^*$, determines the collision-free portion of $\bb{\pi}(t)$. To avoid costly exact computation, an efficient approximation is introduced in the sequel.

A collection of spines emanating from a single configuration is called a $Bur$. Furthermore, a \textit{generalized bur} ($GBur$) \cite{lacevic2020gbur} is constructed by further concatenation of spines along a candidate trajectory (see Fig. \ref{fig_computing_gbur} (a)). 

CFS45 method distinguishes two trajectory types:
\begin{itemize}
	\item \textit{Regular trajectory} (see Subsec. \ref{Subsec. Regular Trajectories}) -- a nominal sequence of quintic splines generated toward the current target waypoint, without prior safety certification. Hence, a collision may theoretically occur while the robot is still moving, corresponding to a type I collision;
	
	\item \textit{Safe trajectory} (see Subsec. \ref{Subsec. Safe Trajectories}) -- an executable trajectory obtained after safety verification by constraining the motion to lie inside a chain of connected \textit{dynamic expanded bubbles of free $\mathcal{C}$-space} ($\mathrm{DEB}$s) and appending a quartic emergency stopping spline to a part of quintic spline. Thus, any collision, should it occur, is constrained to happen only once the robot has stopped, corresponding to a type II collision (see \cite{covic2025realtime} for more details).
\end{itemize}

\vspace{-0.3cm}
\subsection{Regular Trajectories}
\label{Subsec. Regular Trajectories}
Suppose that a spline is given as $\bb{\pi}_\mathrm{reg}(t) = \bb{\pi}[\bb{q}_0,\bb{q}_f]$ for $t\in[t_0,t_f]$, and is called a \textit{local regular trajectory} (e.g., the blue spline in Fig. \ref{fig_computing_gbur} (b)). The proposed procedure given in Alg. \ref{Pseudocode Bur} involves the following steps:

\textbf{Step 1} (lines \ref{Alg3 begin}--\ref{Alg3 line6}): Discretize the spline $\bb{\pi}_\mathrm{reg}(t)$ with a step $\Delta t$ to obtain \textit{intermediate configurations} $\bb{q}_k=\bb{\pi}_\mathrm{reg}(k\Delta t)$, $k=\{1,\dots,N\}$ (e.g., nodes $\bb{q}_1,\dots, \bb{q}_{10}$ in Fig. \ref{fig_computing_gbur} (b)).

\textbf{Step 2} (line \ref{Alg3 bubble}): Compute the bubble $\mathcal{B}(\bb{q}_0, \bb{d}(t_0))\equiv \mathcal{B}_0$ at the root $\bb{q}_0$ using the vector of minimal distances $\bb{d}=\bb{d}(t_0) = [d_1,\dots, d_n]^T$ for each robot's link as proposed in \cite{ademovic2016path}. 

\textbf{Step 3} (lines \ref{Alg3 for}--\ref{Alg3 end}): Simply check which intermediate nodes lie within $\mathcal{B}_0$. If \eqref{eq_bubble} is satisfied, the spine $\overline{\bb{q}_0 \bb{q}_k}$ is collision-free for $t\in[t_0,t_f]$, i.e., $\overline{\bb{q}_0 \bb{q}_k} \in \mathcal{B}_0$, (e.g., spines $\overline{\bb{q}_0 \bb{q}_1},\dots, \overline{\bb{q}_0 \bb{q}_3}$ from Fig. \ref{fig_computing_gbur} (b)). Otherwise, it implies that all nodes $\bb{q}_p \notin \mathcal{B}_0$, $\forall p \in\{k,\dots,N\}$, (e.g., nodes $\bb{q}_4, \dots, \bb{q}_{10}$ in Fig. \ref{fig_computing_gbur} (b)). The collection of spines $\overline{\bb{q}_0 \bb{q}_k} \in \mathcal{B}_0$ comprise a single $Bur$. The described procedure yields an approximate $t^* \approx t_{k-1}$ (see Fig. \ref{fig_computing_gbur} (a)) (e.g., $t^* \approx t_3$ in Fig. \ref{fig_computing_gbur} (b)).

\SetEndCharOfAlgoLine{}
\begin{algorithm}[t]
	\footnotesize
	\setlength{\baselineskip}{0pt}
	
	\SetAlgorithmName{Algorithm}{}
	\Indp
	\caption{Generalized bur -- $\mathtt{computeGBur}$}
	\label{Pseudocode generalized bur}
	\small
	
	\KwIn{$\bb{\pi}_\mathrm{reg}(t) = \bb{\pi}[\bb{q}_0,\bb{q}_f]$, $\Delta t$, $\bb{d}(t_0)$} 
	
	\KwOut{$GBur$}
	
	$GBur \gets \varnothing$,\quad 
	$k\,\gets 0$ \\
	
	\While{$\bb{q}_k \neq \bb{q}_f$}
	{
		$Bur_k \gets \mathtt{computeBur}(\bb{\pi}[\bb{q}_k, \bb{q}_f],\, \Delta t,\, \bb{d}(t_{k}))$ \label{Alg4 computeBur} \\
		
		\If{$Bur_k \neq \varnothing$}
		{\label{Alg4 Q* empty check}
			
			$\overline{\bb{q}_k \bb{q}_m} \gets Bur_k(\textbf{end})$ \label{Alg4 q_m} \\
			
			$t_m \gets \mathtt{getTime}(\bb{q}_m)$ \label{Alg4 getTime} \\
			
			$\bb{d}(t_m) \gets \mathtt{getDistancesToPlanes}(\bb{q}_m,\, \bb{\mathcal{P}}(t_m))$ \label{Alg4 getDistancesToPlanes} \\
			
			$GBur \gets [GBur,\, Bur_k]$, \quad
			$k \gets m$ \label{Alg4 update Q} \\
		}
		\Else
		{
			\Return $GBur$
		}
	}
\end{algorithm}
\setlength{\textfloatsep}{0pt}

\begin{figure*}[t]
	\centering
	\resizebox{0.95\textwidth}{!}{%
\begin{tikzpicture}[
	node distance=2mm,
	>=Latex,
	font=\scriptsize,
	boxs/.style={
		rectangle,
		rounded corners=1.2pt,
		draw,
		align=center,
		inner xsep=2.2pt,
		inner ysep=2.2pt,
		minimum height=9mm
	},
	boxm/.style={
		boxs,
		text width=1.5cm,
		minimum height=9mm
	},
	boxl/.style={
		boxs,
		text width=2cm,
		minimum height=11.5mm
	},
	line/.style={
		draw,
		->,
		ultra thin
	}
	]
	
	\node[boxm, fill=gray!20] (input)
	{$\bb{q}_{\mathrm{curr}}$,\\[0.2mm] 
	$\mathcal{WO}$, $\bb{d}$, $\bb{\mathcal{P}}$};
	
	\node[boxm, fill=gray!20, right=of input] (target)
	{Select/update\\[0.3mm] 
	$\bb{q}_{\mathrm{target}}$};
	
	\node[boxm, fill=blue!20, right=of target] (regular)
	{Generate\\[0.3mm] 
	$\bb{\pi}_{\mathrm{reg}}(t)$};
	
	\node[boxs, right=of regular] (discretize)
	{Discretize $\bb{\pi}_{\mathrm{reg}}(t)$,\\[0.3mm]
	compute $GBur$};
	
	\node[boxs, right=of discretize] (deb)
	{Check\\[0.3mm] 
	$GBur\in\mathrm{DEB}$};
	
	\node[boxl, fill=red!20, right=of deb] (emg)
	{Is $\pi_{\mathrm{emg}}(t)$\\[0.3mm] 
	computed from\\[0.3mm] $\bb{q}_{\mathrm{new}}\in\mathrm{DEB}$?};
	
	\node[boxl, fill=green!20, right=4.7mm of emg] (safe)
	{Form $\bb{\pi}_{\mathrm{safe}}(t)=$\\[0.3mm]
	$\bb{\pi}[\bb{q}_{\mathrm{curr}},\bb{q}_{\mathrm{new}}]$ \\
	$\cup\, \bb{\pi}_{\mathrm{emg}}(t)$};
	
	\node[boxl, fill=gray!20, right=of safe] (exec)
	{Execute\\[0.3mm]
	$\bb{\pi}_{\mathrm{safe}}(t)$ until\\[0.3mm] 
	next iteration};
	
	\draw[line] (input) -- (target);
	\draw[line] (target) -- (regular);
	\draw[line] (regular) -- (discretize);
	\draw[line] (discretize) -- (deb);
	\draw[line] (deb) -- (emg);
	\draw[line] (emg) -- (safe);
	\draw[line] (safe) -- (exec);
	
	\draw[line] (emg.south) -- ++(0,-3mm) node[pos=1.2, above, font=\scriptsize] {~~~~no} -| (target.south);
	\draw[line] (exec.south) -- ++(0,-4.5mm) -| (input.south);
	
	\draw[line] (emg) -- node[midway, above, font=\scriptsize] {yes} (safe);
	
\end{tikzpicture}%
}
	\vspace{-0.1cm}
	\caption{Flowchart of the proposed pipeline. Gray boxes are executed by a dynamic planner, whereas the remaining boxes constitute CFS45.}
	\label{fig_flowchart}
	\vspace{-0.6cm}
\end{figure*}

After reaching the border of the bubble $\mathcal{B}_0$ (e.g., $\bb{q}_3$ in Fig. \ref{fig_computing_gbur} (b)), spines $\overline{\bb{q}_0 \bb{q}_k} \notin \mathcal{B}_0$ do not have to be in a collision (e.g., spines $\overline{\bb{q}_0 \bb{q}_4}, \dots, \overline{\bb{q}_0 \bb{q}_{10}}$ from Fig. \ref{fig_computing_gbur} (b)). Therefore, to further explore free space beyond $\mathcal{B}_0$, Alg. \ref{Pseudocode generalized bur} proposes the following:

\textbf{Step 1} (lines \ref{Alg4 computeBur}--\ref{Alg4 Q* empty check}): For the current root configuration $\bb{q}_k$, compute a bur $Bur_k$ by checking nodes from $\bb{\pi}[\bb{q}_k,\bb{q}_f]$, $\forall k\in\{m,\dots,N-1\}$, for some $m\in\{0,\dots,N-1\}$, for membership in $\mathcal{B}(\bb{q}_k, \bb{d}(\bb{q}_k))$ (e.g., black spines from $\bb{q}_0$ towards $\bb{q}_1,\bb{q}_2,\bb{q}_3 \in \mathcal{B}_0$, and from $\bb{q}_3$ towards $\bb{q}_4,\bb{q}_5,\bb{q}_6 \in \mathcal{B}_3$ in Fig. \ref{fig_computing_gbur} (b)). If $Bur_k$ exists, continue to subsequent steps. Otherwise, return all burs computed so far.

\textbf{Step 2} (lines \ref{Alg4 q_m}--\ref{Alg4 getTime}): Let $\bb{q}_m$ be the last node in the current bur $Bur_k$. Use $\bb{q}_m$ as the root for a new bubble $\mathcal{B}_m$ (e.g., $\bb{q}_3$ and $\bb{q}_6$ become the root of $\mathcal{B}_3$ and $\mathcal{B}_6$ in Fig. \ref{fig_computing_gbur} (b), respectively). 

\textbf{Step 3} (line \ref{Alg4 getDistancesToPlanes}): While computing distance to obstacles $d_{i,j}$, nearest points $R_{i,j}$ and $O_{i,j}$ between $i$-th robot's link and $j$-th obstacle $\mathcal{WO}_j$ can be obtained. They define a separating plane $P_{i,j}$ dividing free/occupied halfspaces as proven in \cite{lacevic2020gbur}. After the robot moves from $\bb{q}_k$ to $\bb{q}_m$, new underestimates of $d_{i,j}$ can be easily acquired as a distance to $P_{i,j}$ when the robot assumes $\bb{q}_m$ (e.g., $\bb{q}_3$ and $\bb{q}_6$ in Fig. \ref{fig_computing_gbur} (c, d)). All $\mathcal{P}_{i,j}$ are stored within a matrix $\bb{\mathcal{P}}$. 

\textbf{Step 4} (line \ref{Alg4 update Q}): Concatenate all spines from each bur $Bur_k$ into a single $GBur$. Repeat steps 1--4 until $\bb{q}_k = \bb{q}_f$ or no new bur can be formed (e.g. $Bur_6$ in Fig. \ref{fig_computing_gbur} (b)).

\vspace{-0.5cm}
\subsection{Safe Trajectories}
\label{Subsec. Safe Trajectories}
For DEs, conditional stopping-safety guarantees via $\mathrm{DEB}$s can be provided under an upper bound on obstacle velocity $v_\mathrm{obs}$, as in \cite{covic2025realtime}. For $\mathrm{DEB}$ construction, each plane $\mathcal{P}_{i,j}$ is conservatively propagated toward the $i$-th robot's link with the worst-case bounded speed $v_\mathrm{obs}$ (see Fig. \ref{fig_computing_gbur} (c)--(e)), which upper-bounds arbitrary obstacle motion rather than assuming a truly constant obstacle velocity. If the true obstacle speed exceeds \(v_\mathrm{obs}\), the formal $\mathrm{DEB}$-based guarantee becomes no longer valid. Thus, when checking whether $\bb{q}(t)\in \mathrm{DEB}$, the distance $d_c$ in \eqref{eq_bubble} is reduced by the plane's traversed path length $v_\mathrm{obs}(t-t_0)$.

Moreover, after reaching the configuration $\bb{q}_\mathrm{new}$, which represents the border of the last generated bubble, \textit{emergency (quartic) spline} $\bb{\pi}_\mathrm{emg}(t) = \bb{\pi}[\bb{q}_\mathrm{new}, \bb{q}_\mathrm{stop}]$ is computed. Thus, the robot can stop in a new configuration $\bb{q}_\mathrm{stop}$ as needed (e.g., the red spline from $\bb{q}_6$ in Fig. \ref{fig_computing_gbur} (b)). If $\bb{\pi}_\mathrm{emg}(t)$ is collision-free, a \textit{local safe trajectory}, defined as $\bb{\pi}_\mathrm{safe}(t) = \bb{\pi}[\bb{q}_\mathrm{curr}, \bb{q}_\mathrm{new}] \cup \bb{\pi}_\mathrm{emg}(t)$, can be followed from the current iteration (e.g., the green trajectory in Fig. \ref{fig_computing_gbur} (b)). Otherwise, the robot can execute emergency stopping from $\bb{q}_\mathrm{curr}$, which was computed at the planner's previous iteration (e.g., stopping at $\bb{q}_\mathrm{stop}'$ depicted by the dashed red line in Fig. \ref{fig_computing_gbur} (b)). 

To summarize, Fig. \ref{fig_flowchart} provides a compact flowchart of the proposed pipeline. Starting from the current robot state and the target waypoint, CFS45 first generates a nominal regular trajectory, extracts its collision-free portion using the $GBur$ construction, embeds it into a chain of $\mathrm{DEB}$s under bounded obstacle motion, and finally outputs a guaranteed safe trajectory augmented with an emergency stopping spline.

\vspace{-0.3cm}
\section{Simulation Study}
\label{Sec. Simulation Study}
This section provides an extensive simulation study\footnote{The implementation of CFS45 method in C++ is available online \href{https://github.com/robotics-ETF/RPMPLv2/tree/main/src/planners/trajectory}{here}. 
Moreover, the incorporation of Ruckig library can be found there in \texttt{TrajectoryRuckig} class. The simulation was performed using the laptop PC with Intel\textregistered\, Core\texttrademark\, i7-9750H CPU @ 2.60 GHz $\times$ 12 with 16 GB of RAM, with the code compiled to run on a single core of the CPU without any GPU acceleration.} with two distinct goals. The first goal is to measure the execution time required for generating ``random'' trajectories during the randomized trial scenarios. Another goal is to conduct a benchmark where we compare our method to the state-of-the-art Ruckig method \cite{berscheid2021ruckig} within two sampling-based planning approaches tailored for DEs. This competitor is selected primarily due to its computational efficiency, code availability, and the additional relevant features from Tab. \ref{tab_sota_summary}.

For planning we employ two approaches: a simple online version of RRT-like algorithm (ORRT), and a more recent DRGBT algorithm \cite{covic2025realtime} which is specifically dedicated to DEs. ORRT tries to extend from $\bb{q}_\mathrm{curr}$ toward $\bb{q}_\mathrm{target}$. Extensions are achieved only when the straight-line connection $\overline{\bb{q}_\mathrm{curr} \bb{q}_\mathrm{target}}$ is determined to be collision-free, where $\bb{q}_\mathrm{target}$ can be equal to $\bb{q}_\mathrm{goal}$ or a randomly sampled configuration $\bb{q}_\mathrm{rand}$. DRGBT is a real-time sampling-based motion planning algorithm for DEs that uses an \textit{adaptive horizon} of prospective target nodes along a preplanned $\mathcal{C}$-space path. By assigning node weights based on relative distances to obstacles and environmental changes, it continuously evaluates whether replanning is required.



\begin{figure}[t]
	\centering
	\includegraphics[width=0.7\linewidth]{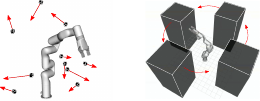}
	\vspace{-0.1cm}
	\caption{Two used scenarios: Scenario 1 (left) -- ten small random obstacles, and Scenario 2 (right) -- four large predefined obstacles. Motion of obstacles is depicted by the red velocity vectors.}
	\label{fig_scenarios}
\end{figure}

\begin{figure*}[t]
	\centering
	\includegraphics[width=0.9\linewidth]{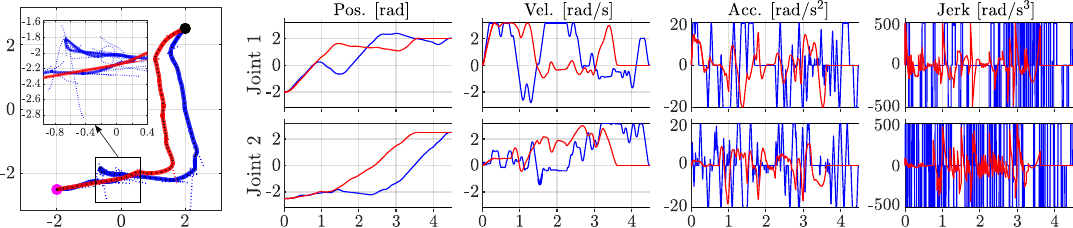}
	\vspace{-0.15cm}
	\caption{An example of generated trajectories (in $\{q_1, q_2\}$-plane on the left and $\mathbb{P}$, $\mathbb{V}$, $\mathbb{A}$, and $\mathbb{J}$ functions vs time (abscissa in [$\mathrm{s}$]) on the right hand side) by both {\color{red}CFS45} and {\color{blue}Ruckig} approaches for a planar 2-DoF robotic manipulator. Black lines depict minimum/maximum values.
	}
	\label{fig_example_traj}
	\vspace{-0.7cm}
\end{figure*}

\vspace{-0.3cm}
\subsection{Scenario Setup}
\label{Subsec. Scenario Setup}
\vspace{-0.05cm}
The simulation study deals with 19 scenario types, with each type using the planner's iteration time $T\in \{1,2,3,\dots,10,20,30,\dots,100\}\, [\mathrm{ms}]$. This time determines the frequency of generating a new trajectory from $\bb{q}_\mathrm{curr}$ toward a (possibly) new $\bb{q}_\mathrm{target}$. Fig. \ref{fig_scenarios} illustrates two scenarios used in the simulation study. The first one consists of ten random obstacles, where each one is assigned a random velocity with its magnitude limited to $1.6\,\mathrm{[\frac{m}{s}]}$. The second scenario uses four large obstacles moving at a random speed up to $0.3\,\mathrm{[\frac{m}{s}]}$.

The model of the UFactory xArm6 robot is exposed to 1000 different simulation runs with randomly generated circumstances (i.e., 1000 random but collision-free start and goal configurations). We opted to use quintic splines ($m=5$) within our approach since the real xArm6 manipulator has to meet constraints $\mathcal{K}$ on maximal joint $\mathbb{V}$, $\mathbb{A}$, and $\mathbb{J}$. These values are: $\bb{\omega}_{max} = \bb{\pi}_{n\times 1}\,\mathrm{[\frac{rad}{s}]}$, $\bb{\alpha}_{max} = \bb{20}_{n\times 1}\,\mathrm{[\frac{rad}{s^2}]}$, and $\bb{j}_{max} = \bb{500}_{n\times 1}\,\mathrm{[\frac{rad}{s^3}]}$, respectively, taken from the xArm6 datasheet. For the sake of completeness, we carry out simulations for both types of trajectories, as defined in Sec. \ref{Sec. Trajectory Collision Checking}.

\vspace{-0.3cm}
\subsection{Comparison with Ruckig}
Fig. \ref{fig_example_traj} shows an example of generated trajectories by both CFS45 and Ruckig approaches for a planar 2-DoF manipulator which is guided by DRGBT algorithm with $T = 50\, [\mathrm{ms}]$. Black dots designate points in each planner's iteration, while dotted red/blue curves depict generated local trajectories in each iteration. The accompanying \href{https://youtu.be/Qwsu-ytHNy8}{video} shows more scenarios operating in real time for regular and safe trajectories generated by both Ruckig and the proposed algorithm.

Fig. \ref{fig_example_traj} also illustrates $\mathbb{P}$, $\mathbb{V}$, $\mathbb{A}$, and $\mathbb{J}$ responses. Clearly, both methods satisfy all kinematic constraints and successfully reach the goal. In terms of smoothness (captured via jerk L1-norm), the proposed approach displays more desirable behavior with the average of 3.3 times improvement. 
Moreover, for the conducted scenarios, the proposed approach shows on average 2.33 times lower Frechet distance \cite{alt1995computing}, which is computed in each iteration as a ``distance'' between the line segment $\overline{\bb{q}_\mathrm{curr} \bb{q}_\mathrm{target}}$ and the geometric path resulting from the computed trajectory $\bb{\pi}[\bb{q}_\mathrm{curr}, \bb{q}_\mathrm{target}]$. An additional comparison with Ruckig based on the cumulative weighted squared norm of the joint velocity vector yielded comparable results and was omitted due to space limitations.

\begin{figure}[t]
	\centering
	\includegraphics[width=0.9\linewidth]{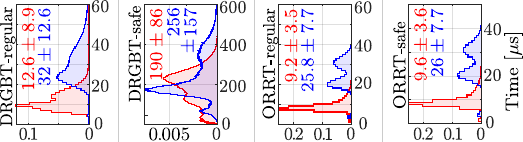}
	\vspace{-0.1cm}
	\caption{PDFs constructed for execution times of generating trajectories using {\color{red}CFS45} and {\color{blue}Ruckig} methods. Mean and standard deviation times are indicated in the corresponding subfigures.}
	\label{fig_histogram}
\end{figure}

Fig. \ref{fig_histogram} depicts probability density functions (PDFs) of execution times required for generating trajectories using CFS45 and Ruckig methods within four dynamic-planning variants: DRGBT-regular/safe and ORRT-regular/safe. Clearly, the proposed method completes almost 3 times faster than Ruckig. It is worth mentioning that each histogram captures more than 24 million different ``random'' trajectories. The computational breakdown analysis has shown that CFS45 accounts for at most $0.1\,[\%]$ and $2.8\,[\%]$ of the considered DRGBT average runtime for regular and safe trajectory generation respectively, indicating that it is not a  computational bottleneck. The case of safe trajectory is more involved, as it augments regular splines with $GBur$ construction, $\mathrm{DEB}$-based certification, and emergency-stopping computation.

Finally, Fig. \ref{fig_results_criteria} reveals the performance of DRGBT-regular/safe and ORRT-regular/safe methods when using CFS45 and Ruckig approaches as trajectory generator. The performance is defined by the following criteria used for the comparison: (adjusted) success rate, algorithm time (required time for the robot to reach the goal), and path length from the start to the goal. Unlike the standard binary success metric (1 if the goal is reached, and 0 otherwise), the \textit{adjusted success} is measured as a proxy for each run using the real number 
$1-\frac{\|\bb{q}_\mathrm{end} - \bb{q}_\mathrm{goal}\|}{\|\bb{q}_\mathrm{start} - \bb{q}_\mathrm{goal}\|} \in [0,1]$, 
where $\bb{q}_\mathrm{end}$ represents an end configuration (in case of collision, it is a configuration at which the collision occurred, otherwise $\bb{q}_\mathrm{end} = \bb{q}_\mathrm{goal}$). After averaging all adjusted successes for each run, we obtain the \textit{adjusted success rate}. Clearly, Fig. \ref{fig_results_criteria} shows that CFS45 outperforms Ruckig according to all criteria. Particularly, considerable performance improvement is achieved at higher planner frequencies (above $100\, \mathrm{[Hz]}$, i.e., for $T \leq 10\, \mathrm{[ms]}$). It is worth stressing that we also conducted an additional comparison with TrajOpt \cite{schulman2014motion}, but its available implementation turned out to be less suitable for the considered real-time setting with frequent target-state updates, where the proposed method showed clear advantages.

\begin{figure}[t]
	\centering
	\includegraphics[width=0.94\linewidth]{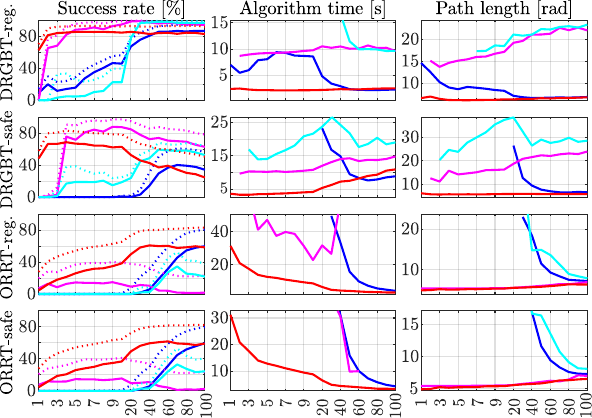}
	\vspace{-0.15cm}
	\caption{The performance of the used dynamic planners vs their iteration time $T$ (abscissa in $\mathrm{[ms]}$) in cases: {\color{red}CFS45 for Scenario 1} and {\color{blue}Ruckig for Scenario 1};  {\color{magenta}CFS45 for Scenario 2} and
		{\color{cyan}Ruckig for Scenario 2}. Dashed line depicts adjusted success rate. Algorithm time and path length are not shown in cases when success rate is less than $10\,\mathrm{[\%]}$ or in cases when they are greater than the shown plot limits.}
	\label{fig_results_criteria}
\end{figure}

\vspace{-0.3cm}
\section{Experimental Validation}
\label{Sec. Experimental Validation}
\vspace{-0.05cm}
To validate the proposed CFS45 method in both static and dynamic environments (i.e., with human presence), we conduct six experiments on the real UFACTORY xArm6 manipulator. Environment sensing is performed using two \textit{Intel RealSense D435i} depth cameras operating at $f_{perc} = 20\,\mathrm{[Hz]}$. The perception pipeline fuses the left and right point clouds into a unified cloud to extract obstacles as axis-aligned bounding boxes. For collision and distance queries, the robot links are approximated by bounding capsules. The low-level controller operates at $f_{con} = 500\,\mathrm{[Hz]}$, handling the desired and measured $\mathbb{P}$ and $\mathbb{V}$ for each joint. The planning algorithm runs in a ROS2 environment (see the implementation \href{https://github.com/robotics-ETF/xarm6-etf-lab}{here}). 

The planner outputs a desired trajectory, containing $\mathbb{P}$, $\mathbb{V}$, and $\mathbb{A}$ for each joint, with frequency $f_{alg} = 1/T$. We use $f_{alg} = 20\,\mathrm{[Hz]}$ to synchronize it with $f_{perc}$. Since $f_{alg}$ generally differs from $f_{con}$, the obtained trajectory is sampled at $f_{con}$ and passed to the controller to achieve smooth motion.

The accompanying \href{https://youtu.be/Qwsu-ytHNy8}{video} shows successful execution of regular and safe trajectories for different jerk limits, e.g., $50\,\mathrm{[\frac{rad}{s^3}]}$ and $200\,\mathrm{[\frac{rad}{s^3}]}$. Representative snapshots are shown in Fig.~\ref{fig_exp_snapshots}, where the end-effector path is depicted by red lines. The measured $\mathbb{P}$ and $\mathbb{V}$ for each joint are given in the same video.
Clearly, all velocities remain within the set limit of $1.5\,\mathrm{[\frac{rad}{s}]}$ for each joint. Unfortunately, the used robot does not provide joint acceleration measurements. 

\begin{figure}[t]
	\centering
	\includegraphics[width=0.95\linewidth]{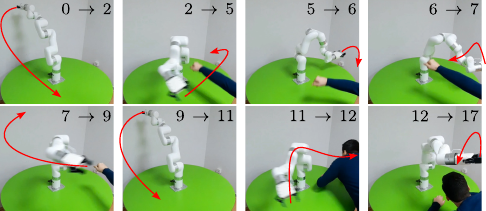}
	\vspace{-0.1cm}
	\caption{Snapshots from the real-world scenario. Time transition in $\mathrm{[s]}$ is depicted in each snapshot.}
	\label{fig_exp_snapshots}
\end{figure}


\vspace{-0.3cm}
\section{Conclusion}
\label{Sec. Conclusion}
We presented an online method for converting arbitrary geometric paths into jerk-limited, kinematically-feasible trajectories using sequences of quintic/quartic splines. The approach accounts for all the kinematic constraints up to limited jerks. The procedure interpolates safe trajectories and provides safety guarantees in both static and dynamic environments. 

Moreover, a comprehensive simulation comparison against the state-of-the-art approach demonstrates significant improvement in performance considering trajectory smoothness, computational efficiency, and dynamic planner's frequency. Experimental validation confirms the proposed approach is suitable for real-world dynamic environments, including experiments in which a moving person is treated conservatively as an external obstacle by the perception and planning pipeline.

Future work will include validation of the novel method on different types of robots (e.g., mobile robots). Moreover, robots with many DoFs will be used to inspect how the proposed approach scales with the increased dimensionality.

\vspace{-0.3cm}


\bibliographystyle{IEEEtran}
\bibliography{../References} 

\end{document}